\documentclass[pmlr,twocolumn,a4paper,10pt]{jmlr}
\usepackage{isipta2021}
\usepackage[american]{babel}
\usepackage{booktabs} 
\usepackage{tikz} 

\usepackage{comment}
\usepackage{enumitem}

\usepackage[mathscr]{euscript}
 \let\mathscr\relax

\newcommand{\edesirs}{\mathcal{E}}
\newcommand{\pspace}{\varOmega}

\newcommand{\domain}{\mathcal{K}}

\newcommand{\gambles}{\mathcal{L}}
\newcommand{\posi}{\mathsf{posi}}

\newcommand{\desirs}{\mathcal{D}}

\newcommand{\partit}{\mathcal{P}}

\title[Information algebras of coherent sets of gambles]{
  Information algebras of coherent sets of gambles \titlebreak~ 
 in general possibility spaces
}

\author{
  \Name{Juerg Kohlas}\Email{juerg.kohlas@unifr.ch}\\ 
  \addr Department of Informatics DIUF, University of Fribourg, Switzerland
  \AND
  \Name{Arianna Casanova}\Email{arianna@idsia.ch}\\
    \Name{Marco Zaffalon}\Email{zaffalon@idsia.ch}\\
  \addr Istituto Dalle Molle di Studi sull’Intelligenza Artificiale (IDSIA), Switzerland
}

\begin{document}
\maketitle

\begin{abstract}
In this paper, we show that coherent sets of gambles can be embedded into the algebraic structure of \emph{information algebra}. This leads firstly, to a new perspective of the algebraic and logical structure of desirability and secondly, it connects desirability, hence imprecise probabilities, to other formalism in computer science sharing the same underlying structure.
Both the \emph{domain-free} and the \emph{labeled} view of the information algebra of coherent sets of gambles are presented, considering general possibility spaces.
\end{abstract}

\section{Introduction and Overview}

Recently \citet{mirzaffalon20} have derived some results about compatibility or consistency of coherent sets of gambles and remarked that these results were in fact about the theory of information or valuation algebras (see \citep{kohlas03}).  

This point of view however, was not worked out in \citep{mirzaffalon20}.  In this and in our previous work \cite{kohlas21} this issue is taken up. We abstract away properties of desirability that can 
be regarded as properties of information algebras rather than special ones of desirability. This is made showing  that, in particular,  coherent  sets  of  gambles augmented with the set of all gambles form an information algebra. A similar scope was pursued by De Cooman in \citep{decooman05}. He discovered indeed that there is a common order-theoretic structure underlying many of the models for representing beliefs in the literature, including lower previsions and sets of almost desirable gambles. Even if they share some elements, the latter focuses more on the study of \emph{belief dynamics} (belief expansion and belief revision). 

From the point of view of information algebras, sets of gambles defined on a possibility space $\Omega$ are indeed pieces of information about certain questions or variables identified by families of equivalence relations $\equiv_x$ on $\Omega$, for $x$ in some index set $Q$. In particular, this paper is intended as an extension of our previous paper, in which we treat the particular case where
information one is interested in concerns the values of certain groups of
variables $\{X_i: \; i \in I\}$ with $I$ an index set, $\Omega = \bigtimes_{i \in I} \Omega_i$, where $\Omega_i$ is the set of possible values of $X_i$, and $\omega \equiv_S \omega' \iff \omega|_S=  \omega'|_S$, \footnote{If we think of $\omega \in \Omega$, as a map $\omega: I \rightarrow \Omega$, $\omega|_S$ is the restriction of the map $\omega$ to $S$.} for every $S \subseteq I$ and $\omega, \omega' \in \Omega$. (see \cite{kohlas21}). 
  
  Such pieces of information can be aggregated and the information they contain about specific questions can be extracted.  This leads to an algebraic structure satisfying a number of simple axioms.  
  There are two different versions of information algebras: a \emph{domain-free} one that correspond to the general treatment of coherent sets of gambles defined on $\Omega$; a \emph{labeled one}, more suitable when gambles implicitly depend only on a specific question. They are closely related and each one can be derived or reconstructed form the other. The domain-free version is better suited for theoretical studies, since it is a structure of universal algebra, whereas the labeled one is better adapted to computational purposes (see \citep{kohlas03}).
  
  In this paper both the views are presented. We begin outlining some preliminaries in Section \ref{sec:DesGambles} and Section \ref{sec:Questions}. In Section \ref{sec:InfAlgs} we derive the domain-free information algebra of coherent sets of gambles. In Section \ref{sec:Atoms} we prove that, in particular, it is an \emph{atomistic} information algebra. Finally, in Section \ref{sec:LabInfAlg}, we derived the labeled version from the domain-free one and in Section \ref{sec:CommExtractors} we analyze the particular case of commuting \emph{extraction} operators that leads to the multivariate case considered in \cite{kohlas21}. 



\section{Desirability} \label{sec:DesGambles}

Consider a set $\Omega$ of possible worlds. A gamble over this set is a bounded function
$f : \Omega \rightarrow \mathbb{R}$.

A gamble is interpreted as an uncertain reward in a linear utility scale. A subject might desire a gamble or not, depending on the information she has about the experiment whose possible outcomes are the elements of $\Omega$.

We denote the set of all gambles on $\Omega$ by $\mathcal{L}(\Omega)$, or more simply by $\mathcal{L}$ when there is no possible ambiguity. We also let $\mathcal{L}^+(\Omega) \coloneqq \{ f \in \mathcal{L}(\Omega): \; f\geq 0, f \not= 0\}$, or simply $\mathcal{L}^+$, denote the subset of non-vanishing, non-negative gambles. These gambles should always be desired, since they may increase the wealth with no risk of decreasing it.
As a consequence of the linearity of our utility scale we assume also that a subject disposed to accept the transactions represented by $f$ and $g$, is disposed to accept also $\lambda f + \mu g$ with $\lambda, \mu \ge 0$ not both equal to $0$.
More generally, we can consider the notion of a coherent set of gambles (see \citep{walley91}).
\begin{definition}[\textbf{Coherent set of desirable gambles}]
We say that a subset $\desirs$ of $\gambles(\pspace)$ is a \emph{coherent} set of desirable gambles if and only if $\desirs$ satisfies the following properties:
\begin{enumerate}[label=\upshape D\arabic*.,ref=\upshape D\arabic*]
\item\label{D1} $\gambles^+ \subseteq \desirs$ [Accepting Partial Gains];
\item\label{D2} $0\notin \desirs$ [Avoiding Null Gain];
\item\label{D3} $f,g \in \desirs \Rightarrow f+g \in \desirs$ [Additivity];
\item\label{D4} $f \in \desirs, \lambda>0 \Rightarrow \lambda f \in \desirs$ [Positive Homogeneity].
\end{enumerate}
\end{definition}
So, $\desirs$ is a convex cone.
This leads to the concept of natural extension.
\begin{definition}[\bf{Natural extension for gambles}] \label{def:natex}Given a set $\domain\subseteq\gambles(\Omega)$, we call $
\edesirs(\domain) \coloneqq\posi(\domain\cup\gambles^+)$,
where
\begin{equation*}
\posi(\domain')\coloneqq\left\{ \sum_{j=1}^{r} \lambda_{j}f_{j}: f_{j} \in \domain', \lambda_{j} > 0, r \ge 1\right\}
\end{equation*}
for every set $\domain' \subseteq \gambles(\Omega)$, its \emph{natural extension}.
\end{definition}
The natural extension $\edesirs(\desirs)$, of a set of gambles $\desirs$, is coherent if and only if $0 \not\in \edesirs(\desirs)$.

Coherent sets are closed under intersection, that is they form a topless $\cap$-structure (see \citep{daveypriestley97}). By standard order theory (see \citep{daveypriestley97}), they are ordered by inclusion, intersection is meet in this order and a join exists if they have an upper bound among coherent sets:
\begin{equation*}
\bigvee_{i \in I} \desirs_i \coloneqq \bigcap \left\{\desirs \in C(\Omega):  \bigcup_{i \in I} \desirs_i \subseteq \desirs \right\},
\end{equation*}
if we denote with $C(\Omega)$, or simply with $C$, the family of coherent sets of gambles on $\Omega$.

Notice also that, if $0 \notin \mathcal{E}(\desirs')$, $\mathcal{E}(\desirs')$ is the smallest coherent set containing $\desirs'$. Therefore, if $\mathcal{E}(\bigcup_{i \in I} \desirs_i)$ is coherent, we have
\begin{equation*}
\bigvee_{i \in I} \desirs_i = \mathcal{E}\left(\bigcup_{i \in I } \desirs_i\right).
\end{equation*}
In view of the following section, it is convenient to add $\mathcal{L}(\Omega)$ to $C(\Omega)$ and let $\Phi(\Omega) \coloneqq C(\Omega) \cup \{\mathcal{L}(\Omega)\}$. The family of sets in $\Phi(\Omega)$, or simply $\Phi$ where there is no possible ambiguity, is still a $\cap$-structure, but now a topped one (see \citep{daveypriestley97}). So, again by standard results of order theory, $\Phi$ is a complete lattice under inclusion, meet is intersection and join is defined for any family of sets $\desirs_i \in \Phi$ as
\begin{equation*}
\bigvee_{i \in I} \desirs_i \coloneqq \bigcap \left\{\desirs \in \Phi: \bigcup_{i \in I} \desirs_i \subseteq \desirs \right\}.
\end{equation*}
Note that, if the family of coherent sets $\desirs_i$ has no upper bound in $C$, then its join is simply $\mathcal{L}(\Omega)$. 

In this topped $\cap$-structure, let us define the following operator
\begin{equation}\label{eq:closureoperatorC}
\mathcal{C}(\desirs') \coloneqq \bigcap \{\desirs \in \Phi: \desirs' \subseteq \desirs\}.
\end{equation}
It can be shown that is a closure (or consequence) operator on subsets of gambles (see \citep{daveypriestley97}).
For further reference, it is easy to prove also the following well-known result.

\begin{lemma} \label{le:UnionClosSets}
For any $\desirs_1, \desirs_2 \subseteq \gambles$ we have:
\begin{equation*}
\mathcal{C}(\mathcal{C}(\desirs_1) \cup \desirs_2) = \mathcal{C}(\desirs_1 \cup \desirs_2).
\end{equation*}
\end{lemma}

Note that $\mathcal{C}(\desirs) = \mathcal{E}(\desirs)$ if $0 \not\in \mathcal{E}(\desirs)$, that is if $\mathcal{E}(\desirs)$ is coherent. Otherwise we may have $\mathcal{E}(\desirs) \not= \mathcal{L}(\Omega)$.  These results prepare the way to an information algebra of coherent sets of gambles (see Section \ref{sec:InfAlgs}). 

The most informative cases of coherent sets of gambles, i.e. coherent sets that are not proper subsets of other coherent sets, are called \textit{maximal}.
\begin{definition}[\textbf{Maximal coherent set of gambles}]
A coherent set of desirable gambles $\desirs$ is maximal if and only if
\begin{equation*}
(\forall f \in \gambles \setminus \{0\})\ f \notin \desirs \Rightarrow -f \in \desirs.
\end{equation*}
\end{definition}
We shall indicate maximal sets of gambles with $M$ to differentiate them from the general case of coherent sets.
These sets play an important role with respect to information algebras (see Section \ref{sec:Atoms}) because of the following facts proved in \citep{CooQua12}:
\begin{enumerate}
\item any coherent set of gambles is a subset of a maximal one;
\item any coherent set of gambles is the intersection of all maximal coherent sets it is contained in.
\end{enumerate}
So far, we have considered sets of gambles in $\mathcal{L}(\Omega)$ relative to a general set of possibilities $\Omega$. 
In the next section, we introduce some structure into it, which allows afterwards to embed coherent sets of gambles into the algebraic structure of information algebras.

%
%
%
%


\section{Stucture of Questions and Possibilities} \label{sec:Questions}

As before, let $\Omega$ be a set of possible worlds. We consider families of equivalence relations $\equiv_x$ for $x$ in some index set $Q$. Informally, we mean that $Q$ represents questions and a question $x \in Q$ has the same answer in possible worlds $\omega$ and $\omega'$, if $\omega \equiv_x \omega'$ (see also \cite{kohlas17}). 

There is a partial order between questions capturing granularity: question $y$ is finer than question $x$ if $\omega \equiv_y \omega'$ implies $\omega \equiv_x \omega'$. This order becomes maybe clearer, if we consider partitions $\partit_x, \; \partit_y$ of $\Omega$ whose blocks are respectively the equivalence classes $[\omega]_x, \; [\omega]_y$ of the equivalence relations $\equiv_x, \; \equiv_y$, representing possible answers to $x$ and $y$.

Then $\omega \equiv_y \omega'$ implies $\omega \equiv_x \omega'$, means that any block $[\omega]_y$ of partition $\partit_y$ is contained in some block $[\omega]_x$ of partition $\partit_x$. The partition of $\Omega$ by $\partit_x$ is coarser than the one by $\partit_y$. 
If this is the case, we say that: $\partit_x \leq \partit_y$.\footnote{In order literature usually the inverse order between partitions is considered. However, this order corresponds better to our natural order of questions by granularity.}

It is well-known that partitions $Part(\Omega)$ of any set $\Omega$ form a lattice under this order. In particular, the partition $\sup\{\partit_x,\partit_y\} = \partit_x \vee \partit_y$ of two partitions $\partit_x, \partit_y$ is, in this order, the partition obtained as the non-empty intersections of blocks of $\partit_x$ with blocks of $\partit_y$. 

We usually assume that the set of questions $Q$ analyzed, considered together with their associated partitions that we denote with $\mathcal{Q} \coloneqq \{\partit_x:x \in Q\}$, is a join-sub-semilattice of $(Part(\Omega),\leq)$ (see \citep{daveypriestley97}). 

As observed before, we may transport the order between partitions to $Q$ and vice versa: $x \leq y$ iff $\partit_x \leq \partit_y$ and we have then $\partit_x \vee \partit_y = \partit_{x \vee y}$ and $\partit_x \wedge \partit_y = \partit_{x \wedge y}$. Furthermore, we assume also often that the top partition in $Part(\Omega)$, i.e. $\partit_\top$ (where the blocks are singleton sets $\{\omega\}$ for $\omega \in \Omega$), belongs to $\mathcal{Q}$.

A gamble $f$ on $\Omega$ is called \emph{$x$-measurable} if for all $\omega \equiv_x \omega'$ we have $f(\omega) = f(\omega')$, that is, if $f$ is constant on every block of $\partit_x$. It could then also be considered as a function (a gamble) on the set of blocks of $\partit_x$.

Let $\mathcal{L}_x(\Omega)$, or $\gambles_x$ when there is no ambiguity, denote the set of all $x$-measurable gambles, so that $\mathcal{L}_\top = \mathcal{L}$. Note that $\mathcal{L}_x$, as well as $\mathcal{L}$, is a linear space for all $x$.

Further, $x \leq y$ if and only if $\mathcal{L}_x$ is a subspace of $\mathcal{L}_y$. So we have $\mathcal{L}_x,\mathcal{L}_y \subseteq \mathcal{L}_{x \vee y}$. In fact $\mathcal{L}_{x \vee y}$ is the smallest subspace containing $\mathcal{L}_x$ and $\mathcal{L}_y$.

Sometimes we want to consider partitions so that $\mathcal{L}_{x \wedge y} = \mathcal{L}_x \cap \mathcal{L}_y$. We show below (Lemma \ref{le:CommRel}) that for this case it is sufficient and necessary that $\omega \equiv_{x \wedge y} \omega'$ implies the existence of an $\omega''$ so that $\omega \equiv_x \omega'' \equiv_y \omega'$.

We consider below coherent sets of gambles as pieces of information, describing beliefs about the likeliness of the possibilities in $\Omega$. However, we may be interested in the content of this information relative to some question $x \in Q$, and we propose how to extract this part of information from the original one. Also, possible beliefs may be originally expressed relative to different questions and these pieces of information must be combined to an aggregated belief. 

This leads then to an algebraic structure, called an information algebra (see \citep{kohlas03}). In the form sketched here, it will be, more precisely, a \emph{domain-free information algebra} (Section \ref{sec:InfAlgs}). Later on, in Section \ref{sec:LabInfAlg}, we consider a \emph{labeled} version of the algebra.
To do this, we first need to introduce a qualitative or logical independence relation between partitions (see \citep{kohlasmonney95,kohlas17}). 
\begin{definition}[Independent Partitions]
For a finite set of partitions $\partit_1,\ldots,\partit_n \in Part(\Omega)$, $n \geq 2$, let us define
\begin{equation*}
R(\partit_1,\ldots,\partit_n) \coloneqq \{(B_1,\ldots,B_n):B_i \in \partit_i,\cap_{i=1}^n B_i \not= \emptyset\}.
\end{equation*}
We call the partitions \emph{independent}, if
\begin{equation*}
R(\partit_1,\ldots,\partit_n) = \partit_1 \times \cdots \times \partit_n.
\end{equation*}
\end{definition}

The intuition behind this definition is the following: $R( \partit_1, \dots, \partit_n)$ contains the tuples of mutually compatible blocks of $\partit_1, \dots, \partit_n$, representing compatible answers to the $n$ questions modelled by the partitions. 
If they are independent, the answer to a question $\partit_i$ does not constrain the answers to the other questions or, in other words, it contains no information relative to the other questions. 

Analogously, we can also introduce a logical conditional independence relation between partitions.
\begin{definition}[Conditionally Independent Partitions]
Consider a finite set of partitions $\partit_1,\dots \partit_n \in Part(\Omega)$, and a block $B$ of a partition $\partit$ (contained or not in the list $\partit_1,\ldots,\partit_n$), then define for $n \geq 1$,
\begin{equation*}
R_B(\partit_1,\ldots,\partit_n) \coloneqq \{(B_1,\ldots,B_n):B_i \in \partit_i,\cap_{i=1}^n B_i \cap B \not= \emptyset\}.
\end{equation*}
We call $\partit_1,\ldots,\partit_n$ \emph{conditionally independent given} $\partit$ if, for all blocks $B$ of $\partit$,
\begin{equation*}
R_B(\partit_1,\ldots,\partit_n) = R_B(\partit_1) \times \cdots \times R_B(\partit_n).
\end{equation*}
\end{definition}
So, $\partit_1,\dots,\partit_n$ are conditionally independent given $\partit$ if knowing an answer to $\partit_i$ compatible with $B \in \partit$, gives no information on the answers to the other questions, except that they must each be compatible with $B$. Note that this relation holds if and only if $B_i \cap B \not= \emptyset$ for all $i=1,\ldots,n$, imply that $B_1 \cap \ldots \cap B_n \cap B \not= \emptyset$. In this case we write $\bot\{\partit_1,\ldots,\partit_n\} \vert \partit$, or, for $n =2$, $\partit_1 \bot \partit_2 \vert \partit$. We may also say that  $\partit_1 \bot \partit_2 \vert \partit$ if and only if $\omega \equiv_{\partit} \omega'$, implies that there is an element $\omega'' \in \Omega$ such that $\omega \equiv_{\partit_1 \vee \partit} \omega''$ and $\omega' \equiv_{\partit_2 \vee \partit} \omega''$.

The three-place relation $\partit_1 \bot \partit_2 \vert \partit$ among partitions has the following properties (see \citep{kohlas17}):

\begin{theorem} \label{th:QSepOfPart}
Given $\partit,\partit',\partit_1,\partit_2 \in Part(\Omega)$, we have:
\
\begin{description}
\item[C1] $\partit_1 \bot \partit_2 \vert \partit_2$;
\item[C2] $\partit_1 \bot \partit_2 \vert \partit$ implies $\partit_2 \bot \partit_1 \vert \partit$;
\item[C3] $\partit_1 \bot \partit_2 \vert \partit$ and $\partit' \leq \partit_2$ impliy $\partit_1 \bot \partit' \vert \partit$;
\item[C4] $\partit_1 \bot \partit_2 \vert \partit$ implies $\partit_1 \bot \partit_2 \vee \partit \vert \partit$.
\end{description}
\end{theorem}

A three-place relation like $\partit_1 \bot \partit_2 \vert \partit$ satisfying C1 to C4 has been called a \textit{quasi-separoid} (q-separoid) in \citep{kohlas17}. It is a reduct of a separoid, a concept discussed in \citep{dawid01} in relation to the concept of (logical) conditional independence in general.
Notice in particular that, thanks to these properties, we have:
\begin{equation*}
    \partit_x \perp \partit_y \vert \partit_z \iff \partit_{x \vee z} \perp \partit_{y \vee z} \vert \partit_z.
\end{equation*}
We will use this property very often later on.

To simplify notation, in what follows, we write: $x \bot y \vert z$ for $\partit_x \bot \partit_y \vert \partit_z$; $x \leq y$, as above, for $\partit_x \leq \partit_y$; $x \vee y$ for $\partit_x \vee \partit_y$.

Join of two partitions $\partit_x$ and $\partit_y$ is simple to define as we have seen above. For meet, $\partit_x \wedge \partit_y$, the situation is different, indeed its definition is somewhat involved (see \citep{graetzer03}).
There is however an important particular case, where meet is also simple. 
\begin{definition}[\textbf{$\star$ product}]
Given two partitions $\partit_x, \partit_y \in Part(\Omega)$, we define the $\star$ product of the correspondent equivalence relations $\equiv_x, \; \equiv_y$ respectively, as:
\begin{equation*}
\equiv_x \star \equiv_y\ \coloneqq \{(\omega,\omega'):\exists \omega'' \textrm{ so that}\ \omega \equiv_x \omega'' \equiv_y \omega'\}.
\end{equation*}
\end{definition}
The following lemma gives a necessary and sufficient condition for it to define an equivalence relation.
\begin{lemma}\label{le:CommRel}
Given two partitions $\partit_x,\partit_y \in Part(\Omega)$, the $\star$ product of the correspondent equivalence relations $\equiv_x, \; \equiv_y$ respectively, is an equivalence relation if and only if:
\begin{equation*}
\equiv_x \star \equiv_y =\equiv_y \star \equiv_x.
\end{equation*}
\end{lemma}
If $\equiv_x$ and $\equiv_y$ commute, then the partition associated with their $\star$-product is the meet of of the associated partitions $\partit_x$ and $\partit_y$ respectively, so that we may write $\equiv_x \star \equiv_y = \equiv_{x \wedge y}$. The partitions are then called commuting and since their meet is defined by $\equiv_x \star \equiv_y$, they are also called Type I partitions (see \citep{graetzer03}).
\begin{definition}[\textbf{Type I partitions/Commuting partitions}]
Two partitions $\partit_x, \partit_y \in Part(\Omega)$ are called \emph{Type I} or \emph{commuting partitions} if the product $\equiv_x \star \equiv_y$
is an equivalence relation. 
\end{definition}

As a consequence, for commuting partitions $\partit_x$ and $\partit_y$ we have also $\mathcal{L}_x \cap \mathcal{L}_y = \mathcal{L}_{x \wedge y}$ and vice versa, as stated already above.

For commuting partitions, the conditional independence relation can also be expressed simply in terms of joins and meets.

\begin{theorem} \label{th:CommCondIndep}
Given $\partit_1,\partit_2, \partit \in Part(\Omega)$, we have
\begin{equation*}
\partit_1 \bot \partit_2 \vert \partit \Leftrightarrow (\partit_1 \vee \partit) \wedge (\partit_2 \vee \partit)=\partit
\end{equation*}
if and only if $\partit_1$ and $\partit_2$ commute.
\end{theorem}

An important instance of such commutative partitions is given in multivariate possibility sets. Let $X_i$ a variable for $i$ in some index set $I$ (usually a finite or countable set) and $\Omega_i$ the set of its possible values. If then
\begin{equation*}
\Omega = \bigtimes_{i \in I} \Omega_i
\end{equation*}
is the set of possibilities, we may think of its elements $\omega$ as maps $\omega : I \rightarrow \Omega$ such that $\omega(i) \in \Omega_i$. If $S$ is any subset of variables, $S \subseteq I$, then let
\begin{equation*}
\Omega_S = \bigtimes_{i \in S} \Omega_i.
\end{equation*}
Further let $\omega \equiv_S \omega'$ if $\omega$ and $\omega'$ coincide on $S$. This is an equivalence relation in $\Omega$ and it determines a partition $\partit_S$ of $\Omega$. 

These partitions commute pairwise. Taking the subsets $S$ of $I$ as index set, according to Theorem \ref{th:CommCondIndep}, we have that $S \bot T \vert R$ (meaning $\partit_S \bot \partit_T \vert \partit_R$) if and only if $(S \cup R) \cap (T \cup R) = R$. Here, the underlying lattice of subsets of $I$ or the corresponding sub-lattice of partitions is distributive.  Then some properties in addition to C1 to C4 hold, making it a \emph{strong separoid} (see \citep{dawid01}).



\section{Information Algebra of Coherent Sets of Gambles} \label{sec:InfAlgs}

We define now on $\Phi(\Omega) = C(\Omega) \cup \{\mathcal{L}(\Omega)\}$, the operations of combination, capturing aggregation of pieces of belief, and extraction, describing filtering the part of information relative to a question $x \in Q$. 
More formally, given $\desirs, \desirs_1,\desirs_2 \in \Phi$ and $x \in Q$, we define:
\begin{enumerate}
\item Combination: $\desirs_1 \cdot \desirs_2 \coloneqq \mathcal{C}(\desirs_1 \cup \desirs_2)$;
\item Extraction: $\epsilon_x(\desirs) \coloneqq \mathcal{C}(\desirs \cap \mathcal{L}_x)$.
\end{enumerate}
Note that $\mathcal{L}$ and $\gambles^+$ are respectively the null and the unit elements of combination, since for every $\desirs \in \Phi$, $\mathcal{C}(\desirs \cup \gambles) = \mathcal{L}$ 
and $\mathcal{C}(\desirs \cup \gambles^+) =\desirs$. The null element signals contradiction, it destroys any other piece of information. The unit or neutral element represents vacuous information, it changes no other piece of information.
To simplify notation, in what follows, we denote the null and the unit element respectively by $0$ and $1$.

Then $(\Phi,\cdot)$ is a commutative, idempotent semigroup. In an idempotent, commutative semigroup, a partial order is defined by $\desirs_1 \leq \desirs_2$ if $\desirs_1 \cdot \desirs_2 = \desirs_2$. Then $\desirs_1 \leq \desirs_2$ if and only if $\desirs_1 \subseteq \desirs_2$. This order is called \textit{information  order}, since $\desirs_1 \leq \desirs_2$ means that $\desirs_1$ is less informative than $\desirs_2$. In this order, the combination $\desirs_1 \cdot \desirs_2$ is the supremum or join of $\desirs_1$ and $\desirs_2$, since $\Phi$ is a lattice,
\begin{equation*}
\desirs_1 \cdot \desirs_2 = \desirs_1 \vee \desirs_2.
\end{equation*}
Note that $\epsilon_x(\desirs) \leq \desirs$ and also $\desirs_1 \leq \desirs_2$ implies $\epsilon_x(\desirs_1) \leq \epsilon_x(\desirs_2)$.

We state now two fundamental theorems about the extraction operator.

\begin{theorem} \label{th:ExQuant}
For any $\desirs,\desirs_1,\desirs_2 \in \Phi$  and $x \in Q$,  we have:
\begin{enumerate}
\item $\epsilon_x(0) = 0$,
\item $\epsilon_x(\desirs) \cdot \desirs =\desirs$,
\item $\epsilon_x(\epsilon_x(\desirs_1) \cdot \desirs_2) = \epsilon_x(\desirs_1) \cdot \epsilon_x(\desirs_2)$.
\end{enumerate}
\end{theorem}
This result can be proven analogously to the correspondent result for multivariate possibility sets shown in \cite{kohlas21}.
An operator on an ordered structure satisfying the condition of this theorem is called an \textit{existential quantifier} in algebraic logic. \footnote{In algebraic logic the ordered structure is mostly a Boolean lattice (see \cite{daveypriestley97}) and the inverse order is considered, so that join becomes meet.}
%
%
Furthermore, we have the following result about conditional independence and extraction.

\begin{theorem} \label{th:CombExtrProp}
Given $x,y,z \in Q$ and $\desirs \in \Phi$. If $x \vee z \bot y \vee z \vert z$ and $\epsilon_x(\desirs) = \desirs$, then:
\begin{equation*}
\epsilon_{y \vee z}(\desirs) = \epsilon_{y \vee z}(\epsilon_z(\desirs)).
\end{equation*}
\end{theorem}

A question $x$, or a partition $\partit_x$, is called a \emph{domain} or \emph{support} of an element $\desirs$ of $\Phi$, if $\epsilon_x(\desirs) = \desirs$. If $\desirs \in C(\Omega)$, it means that the coherent set $\desirs$ refers already to the question $x$. 
The finest top partition of $\Omega$ (all blocks consist of exactly one element $\omega \in \Omega$), is a support of all sets $\desirs \in \Phi$ (including the unit and the null). However, there may be other supports. The following result says that any partition, finer than a support, is also a support.

\begin{proposition}\label{prop:finersupports}
Given $x \in Q$ and $\desirs \in \Phi$. If $x$ is a support of $\desirs$, then any $y \geq x, \; y \in Q$ is also a support of $\desirs$.
\end{proposition}
Here we summarize the algebraic system of $\Phi$ together with a system of questions $Q$ and a family $E$ of extraction operators $\epsilon_x : \Phi \rightarrow \Phi$ for $x \in Q$:
\begin{enumerate}
\item \textit{Semigroup:} $(\Phi,\cdot)$ is a commutative semigroup with a null element $0$ and a unit $1$.
\item \textit{Quasi-Separoid:} $(Q,\leq)$ is a join semilattice and $x \bot y \vert z$ with $x,y,z \in Q$, a quasi-separoid.
\item \textit{Existential Quantifier:} For any $x \in Q$, $\desirs_1,\desirs_2,\desirs \in \Phi$:
\begin{enumerate}
\item $\epsilon_x(0) = 0$,
\item $\epsilon_x(\desirs) \cdot \desirs = \desirs$,
\item $\epsilon_x(\epsilon_x(\desirs_1) \cdot \desirs_2) = \epsilon_x(\desirs_1) \cdot \epsilon_x(\desirs_2)$.
\end{enumerate}
\item \textit{Extraction:} For any $x,y,z \in Q$, $\desirs \in \Phi$, such that $x \vee z \bot y \vee z \vert z$ and $\epsilon_x(\desirs) = \desirs$, we have:
\begin{equation*}
\epsilon_{y \vee z}(\desirs) = \epsilon_{y \vee z}(\epsilon_z(\desirs)).
\end{equation*}
\item \textit{Support:} For any $\desirs \in \Phi$ there is an $x \in Q$ so that $\epsilon_x(\desirs) = \desirs$ and for all $y \geq x, \; y \in Q$, $\epsilon_y(\desirs) = \desirs$.
\end{enumerate}
The existence of a support $x$ for any $\desirs$ is essential for the existence of an associated labeled version of the information algebra (see Section \ref{sec:LabInfAlg}).
The extraction property is an important axiom relating to conditional (logical) independence: if the partition $x \vee z$ is conditionally independent from $y \vee z$ given $z$, then, when from a piece of information regarding (supported by) $x \vee z$ the information relating to $y \vee z$ is extracted, only the part bearing on  $z$ is relevant.

An algebraic system satisfying these conditions is called a \emph{domain-free information algebra}. Therefore, with a little abuse of notation, in what follows we will refer to $\Phi$ as the \emph{domain-free information algebra} of coherent sets of gambles.
This axiomatic system is more general than the usual ones (see \citep{kohlas03,shafershenoy90}). In Section \ref{sec:CommExtractors}, it will be shown that these older axiomatic systems are special cases of the present one. In the unpublished text \citep{kohlas17} a similar system has been proposed and analyzed. 

For further reference, a number of elementary consequences of the axioms above are collected.

\begin{lemma} \label{le:SuppotProp}
Given $x,y \in Q$ and $\desirs, \desirs_1,\desirs_2 \in \Phi$, we have:
\begin{enumerate}
\item $\epsilon_x(1) = 1$,
\item $\epsilon_x(\desirs) = 0$ if and only if $\desirs = 0$,
\item $x$ is a support of $\epsilon_x(\desirs)$,
\item if $x \leq y$, then $\epsilon_x(\desirs) \leq \epsilon_y(\desirs)$,
\item if $x \leq y$, then $\epsilon_y(\epsilon_x(\desirs)) = \epsilon_x(\desirs)$,
\item if $x \leq y$, then $\epsilon_x(\epsilon_y(\desirs)) = \epsilon_x(\desirs)$,
\item if $x$ is a support of both $\desirs_1$ and $\desirs_2$, then it is a support for $\desirs_1 \cdot \desirs_2$,
\item if $x$ is a support of $\desirs_1$ and $y$ a support of $\desirs_2$, then $x \vee y$ is a support for $\desirs_1 \cdot \desirs_2$ and $\desirs_1 \cdot \desirs_2 = \epsilon_{x \vee y}(\desirs_1) \cdot \epsilon_{x \vee y}(\desirs_2)$.
\end{enumerate}
\end{lemma}

These are important properties, especially in view of the labeled version of an information algebra, see Section \ref{sec:LabInfAlg}. Here follow two important generalizations of the Extraction and the Existential quantification axioms, which show the equivalence between the axiomatic definition of a domain-free information algebra given here with the one in \citep{kohlas17}.

\begin{theorem} \label{th:GenCondIndep}
Let $\desirs_1,\desirs_2$ and $\desirs$ be elements of $\Phi$ and $x,y,z \in Q$, such that $x \bot y \vert z$. Then
\begin{enumerate}
\item if $x$ is a support of $\desirs$,
\begin{equation*}
\epsilon_y(\desirs) = \epsilon_y(\epsilon_z(\desirs)).
\end{equation*}
 \item If $x$ is a support of $\desirs_1$ and $y$ of $\desirs_2$,
 \begin{equation*}
\epsilon_z(\desirs_1 \cdot \desirs_2) = \epsilon_z(\desirs_1) \cdot \epsilon_z(\desirs_2).
\end{equation*}
\end{enumerate}
\end{theorem}

Recall that $\Phi$ forms a lattice under information order or inclusion. It turns out that extraction commutes with intersection, i.e. meet in the lattice.

\begin{theorem} \label{th:ExtrCommMeet}
Let $\desirs_j$ with $j \in J$ be any family of elements of $\Phi$ and $x \in Q$. Then
\begin{equation} \label{eq:ExtrCommMeet}
\epsilon_x\left( \bigcap_{j \in J} \desirs_j \right)  = \bigcap_{j \in J} \epsilon_x(\desirs_j).
\end{equation}
\end{theorem}

This result can be proven analogously to the correspondent result for multivariate possibility sets shown in \cite{kohlas21}.

An information algebra like $\Phi$, where $(\Phi,\leq)$ is a lattice under information order and satisfies (\ref{eq:ExtrCommMeet}), is called a \textit{lattice information algebra} (see \citep{kohlasschmid13}).

 In the next section we will see moreover that the information algebra of coherent sets of gambles is also an \emph{atomistic} information algebra.




\section{Atoms} \label{sec:Atoms}


In  certain  information  algebras  there  are  maximally  informative  elements,  called \emph{atoms}.
In terms of the information algebra $\Phi$, we have:
\begin{equation*}
M \le \desirs \textrm{ for } \desirs \in \Phi \iff \desirs= M \textrm{ or } \desirs=0,
\end{equation*}
 if $M$ is a maximal set of gambles. Clearly we have also $M \neq 0$. These are the characterizing properties of \emph{atoms}, therefore, 
maximal sets $M$ are atoms in the information algebra $\Phi$.
This is a well-know concept in information algebras (see \citep{kohlas03}). For example, the following are elementary properties of atoms, immediately derivable from the definition. If $M,M_1$ and $M_2$ are atoms of $\Phi$ and $\desirs \in \Phi$, then:
\begin{enumerate}
\item $M \cdot \desirs = M$ or $M \cdot \desirs = 0$,
\item either $\desirs \leq M$ or $M \cdot \desirs = 0$,
\item either $M_1 = M_2$ or $M_1 \cdot M_2 = 0$.
\end{enumerate}

Let $At(\Phi)$ denote the set of all atoms of $\Phi$ (maximal sets of $\Phi$). Moreover, for any $\desirs \in \Phi$ such that $\desirs \not= 0$, let $At(\desirs)$ denote the set of all atoms $M$ which dominate $\desirs$, that is: \begin{equation*}
    At(\desirs) \coloneqq \{M \in At(\Phi):\desirs \subseteq M\}.
\end{equation*}

In general such sets may be empty.  Not so in the case of coherent sets of gambles. In the case of the information algebra of coherent sets of gambles, we have in fact a number of additional properties concerning atoms:
\begin{itemize}
    \item for any set $\desirs \in C(\Omega)$, $At(\desirs)$ is not empty. An information algebra that satisfies this property is called \textit{atomic} (see \citep{kohlas03}).
    \item For any
    set $\desirs \in C(\Omega)$,
    \begin{equation*}
\desirs = \bigcap At(\desirs),
\end{equation*}
that is, any coherent set of gambles is the infimum, in information order, of the atoms it is contained in. An atomic information algebra which satisfies this additional condition is called \emph{atomic composed} (see \citep{kohlas03}) or \textit {atomistic} (see \citep{kohlasschmid14,kohlasschmid16}). 
\item For any subset $A$ of $At(\Phi)$, we have also that:
\begin{equation*}
    \bigcap A = \desirs \in C(\Omega),
\end{equation*}
so that $A \subseteq At(\desirs)$. In general however $A \not= At(\desirs)$.
\end{itemize}


With a general result of atomistic information algebras, we show that the subalgebras $\epsilon_x(\Phi)$ for $x \in Q$ are also atomistic. 

\begin{theorem}\label{th:epsilonX_atomistic}
For any $x \in Q$, the subalgebra $\epsilon_x(\Phi)$ is also atomistic and $At(\epsilon_x(\Phi)) = \epsilon_x(At(\Phi)) = \{\epsilon_x(M):M \in At(\Phi)\}$.
\end{theorem}
 
We call $\epsilon_x(M)$ for $M \in At(\Phi)$ local atoms for $x$. Indeed, they represent maximally informative pieces of information for question $x$ (or partition $\partit_x$). 

For multivariate possibility sets, it is well known that atomistic information algebras can be embedded into an information algebra of subsets of $At(\Phi)$, a so-called set or relational algebra \citep{kohlas03}, see also \cite{kohlas21}. This is an important representation theorem, since it establishes a link of information algebras with a Boolean structure. The result extends to so-called commutative information algebras (see \citep{kohlasschmid16}) and is expected to hold also for the general case considered here, a result yet to be established.

Up to now we concentrate ourselves on the domain-free view of the information algebra of coherent sets of gambles. In the next section we will derive a labeled version of it. It should be clear moreover that all the results shown in this section could equivalently be expressed in this labeled view.

\section{Labeled Information Algebras} \label{sec:LabInfAlg}

The domain-free view of information algebras treats the general case of coherent sets of gambles defined on $\Omega$. However, it is well known that, if a coherent set of gambles has support $x$, it is essentially determined by values of gambles defined on smaller sets of possibilities than $\Omega$, namely on frames representing the possible answers to the question $x$. 

Indeed, if a coherent set of gambles $\desirs$ has support $x$, it means that $\desirs= \mathcal{C}( \desirs \cap \gambles_x)$. Therefore, it contains the same information of the set $\desirs \cap \gambles_x$ that is in a one-to-one correspondence with a set $\desirs'$ directly defined on 
blocks of $\partit_x$ (see for example \citep{mirzaffalon20}).
This view leads to another, so-called labeled version of an information algebra that clearly is better suited for computational purposes. 

We start deriving a labeled view of the information algebra of coherent sets of gambles, using a general method for  domain-free information algebras to derive corresponding labeled ones (see \citep{kohlas03}). 
From a labeled algebra, the domain-free view may be reconstructed. So the two are equivalent (see \cite{kohlas03, kohlas17}).  
Consider therefore the domain-free information algebra $\Phi$ on $\Omega$ relative to a set $Q$ of questions, represented by the family $\mathcal{Q} = \{\partit_x:x \in Q\}$ of partitions, as described in Section \ref{sec:InfAlgs}. Let then $\Psi_x(\Omega)$, or $\Psi_x$ when there is no ambiguity, denote the set of all pairs $(\desirs,x)$, where $\desirs \in \Phi$ and $x \in Q$ is a support of $\desirs$. Consider then the set 
\begin{equation*}
\Psi(\Omega) \coloneqq \bigcup_{x \in Q} \Psi_x(\Omega),
\end{equation*}
also denoted with $\Psi$
when no confusion is possible.
In this set we define the following operations. Given $x,y \in Q$,  $\desirs, \desirs_1, \desirs_2 \in \Phi$:

\begin{enumerate}
\item \textit{Labeling:} $d(\desirs,x) \coloneqq x$.
\item \textit{Combination:} $(\desirs_1,x) \cdot (\desirs_2,y) \coloneqq (\desirs_1 \cdot \desirs_2,x \vee y)$, where $\desirs_1 \cdot \desirs_2$ is the combination in $\Phi$.
\item \textit{Transport:}  $t_y(\desirs,x) \coloneqq (\epsilon_y(\desirs),y)$, where $\epsilon_y$ denotes the extraction operator in $\Phi$.
\end{enumerate}

It is straightforward to verify the following properties of these operations, given $x,y,z \in Q$, $\desirs, \desirs_1,\desirs_2 \in \Phi$:

\begin{enumerate}
\item \textit{Semigroup:} $(\Psi,\cdot)$ is a commutative semigroup.
\item \textit{Quasi-Separoid} $(Q,\leq)$ is a join semilattice and $x \bot y \vert z$ a quasi-separoid in $Q$.
\item \textit{Labeling} $d((\desirs_1,x) \cdot (\desirs_2,y)) = x \vee y$, $d(t_y(\desirs,x)) = y$.
\item \textit{Unit and Null} For all $x \in Q$, $(1,x) \cdot (1,y) = (1,x \vee y)$,
$(\desirs,x) \cdot (1,x) = (\desirs,x)$,
$(\desirs,x) \cdot (0,x) = (0,x)$ and $t_y(\desirs,x) = (0,y)$ if and only if $(\desirs,x) = (0,x)$, $(\desirs,y) \cdot (1,x)=t_{x \vee y}(\desirs,y)$.
\item \textit{Transport:} $x  \bot y  \vert z$ implies $t_y(\desirs,x) = t_y(t_z(\desirs,x ))$.
\item \textit{Combination:} $x \perp y \vert z$ implies $t_z((\desirs_1,x) \cdot (\desirs_2,y))=
t_z(\desirs_1,x) \cdot t_z(\desirs_2,y)$.
\item \textit{Identity:} $t_x(\desirs,x)=(\desirs,x)$.
\item \textit{Idempotency:} If $y \leq x$, then $t_y(\desirs,x) \cdot (\desirs,x) = (\desirs,x)$.
\end{enumerate}

We take these statements as the axioms of a labeled information algebra, in this case the labeled algebra associated with its domain-free alter ego. 




According to considerations made above, we claim that we might equivalently work with elements $\desirs \cap \mathcal{L}_x$, where $\desirs$ is a coherent set of gambles on $\mathcal{L}(\Omega)$. 

We may therefore define a labeled algebra based on elements $\desirs \cap \mathcal{L}_x$. 
As we have previously noticed, the main advantage of this reformulation is that we can then think of them to be composed by gambles directly defined on blocks of $\partit_x$. This is essential for computational purposes.

So, let us define $\tilde{\Psi}_x(\Omega)$, or more simply $\tilde{\Psi}_x$, to be the family of all sets $(\desirs \cap \mathcal{L}_x,x)$, where $\desirs \in \Phi, \; x \in Q$. Let then
\begin{equation*}
\tilde{\Psi}(\Omega) \coloneqq \bigcup_{x \in Q} \tilde{\Psi}_x(\Omega),
\end{equation*}
also indicated with $\tilde{\Psi}$ when no ambiguity is possible.
In $\tilde{\Psi}$ we define the following operations. Given $(\desirs \cap \gambles_x,x), (\desirs_1 \cap \gambles_x,x), (\desirs_2 \cap \gambles_y,y) \in \tilde{\Psi}$:
\begin{enumerate}
\item \textit{Labeling:} $d(\desirs \cap \mathcal{L}_x,x) \coloneqq x$.
\item \textit{Combination:} $(\desirs_1 \cap \mathcal{L}_x,x) \cdot (\desirs_2 \cap \mathcal{L}_y,y) \coloneqq ((\mathcal{C}(\desirs_1 \cap \mathcal{L}_x) \cdot \mathcal{C}(\desirs_2 \cap \mathcal{L}_y)) \cap \mathcal{L}_{x \vee y}, x \vee y)$, where $\cdot$ denotes the combination operator in $\Phi$.
\item \textit{Transport:} $t_y(\desirs \cap \mathcal{L}_x,x) \coloneqq (\mathcal{C}(\desirs  \cap \mathcal{L}_x)   \cap \mathcal{L}_y,y)$.
\end{enumerate}
Note that we denote combination and transport in $\Psi$ and $\tilde{\Psi}$ by the same symbol; it will always be clear from the context which one is meant.

It is easy to verify that the map $(\desirs,x) \mapsto (\desirs \cap \gambles_x,x)$ from $\Psi$ to $\tilde{\Psi}$ preserves combination and transport, namely
\begin{align*}
    (\desirs_1,x) \cdot (\desirs_2,y) &= (\desirs_1 \cdot \desirs_2,x \vee y) \mapsto \\
((\desirs_1 \cdot \desirs_2) \cap \mathcal{L}_{x \vee y}, x \vee y) &= (\desirs_1 \cap \mathcal{L}_x, x) \cdot (\desirs_2 \cap \mathcal{L}_y,y),
\end{align*}
because $\desirs_1= \mathcal{C}(\desirs_1 \cap \gambles_x)$ and $\desirs_2= \mathcal{C}(\desirs_2 \cap \gambles_y)$ respectively. Moreover:
\begin{align*}
    t_y(\desirs,x) &= (\epsilon_y(\desirs),y) \mapsto \\
  (\epsilon_y(\desirs) \cap \mathcal{L}_y,y) &= t_y(\desirs \cap \mathcal{L}_x,x),  
\end{align*}
because $\desirs= \mathcal{C}(\desirs \cap \gambles_x)$ and $\epsilon_y(\desirs) \cap \gambles_y =\desirs \cap \gambles_y$, indeed $ \desirs \cap \gambles_y \subseteq \mathcal{C}(\desirs \cap \gambles_y) \cap \gambles_y  = \epsilon_y(\desirs) \cap \gambles_y \subseteq \desirs \cap \gambles_y$.
This implies that the axioms to which $\Psi$ is submitted carry over to $\tilde{\Psi}$, which thereby becomes a labeled information algebra. 
This map between $\Psi$ and $\tilde{\Psi}$ is also bijiective, hence an isomorphism. Indeed:
\begin{itemize}
    \item the map is \emph{surjective}. Consider $(\desirs \cap \gambles_x,x) \in \tilde{\Psi}$, 
    then $(\mathcal{C}(\desirs  \cap \gambles_x), x) \mapsto (\mathcal{C}(\desirs  \cap \gambles_x)\cap \gambles_x,x)$ and $\mathcal{C}(\desirs  \cap \gambles_x)\cap \gambles_x = \desirs \cap \gambles_x$. 
    
    \item The map is \emph{injective}. Suppose to have  $(\desirs,x) \mapsto (\desirs \cap \mathcal{L}_x,x)$
    and $(\desirs',y) \mapsto (\desirs' \cap \mathcal{L}_y,y)$
    such that
    $(\desirs \cap \mathcal{L}_x,x)= (\desirs' \cap \mathcal{L}_y,y)$.
    Then, first of all $x=y$ and $\desirs= \epsilon_x(\desirs) = \mathcal{C}(\desirs \cap \mathcal{L}_x) = \mathcal{C}(\desirs' \cap \mathcal{L}_x) = \epsilon_x(\desirs')= \desirs'$.
\end{itemize}

We have proved the following theorem.

\begin{theorem}
The labeled information algebras $\Psi$ and $\tilde{\Psi}$ are isomorphic.
\end{theorem}



\section{Commutative Extractors} \label{sec:CommExtractors}

We consider in this section lattices $\mathcal{Q}$ of \textit{commuting} partitions. This covers for instance the important case of multivariate possibility sets, see the end of Section \ref{sec:Questions}.
Recall that this implies that $\partit_x \bot \partit_y \vert \partit_z$ if and only if $(\partit_x \vee \partit_z) \wedge (\partit_y \vee \partit_z) = \partit_z$, for every $\partit_x,\partit_y\partit_z \in \mathcal{Q}$ (Theorem \ref{th:CommCondIndep}). 
The main effect of this is that all extraction operators commute under composition, whereas, in general, this is only the case if $x \leq y$, see items 5 and 6 of Lemma \ref{le:SuppotProp}. 
Therefore, we can show the following result.
\begin{proposition} \label{prop:CommExtrOp}
If $\mathcal{Q}$ is a sublattice of $(Part(\Omega),\leq)$ of commuting partitions, then for all $x,y,z \in Q$,
\begin{equation*}
\epsilon_x \circ \epsilon_y = \epsilon_y \circ \epsilon_x = \epsilon_{x \wedge y}.
\end{equation*}
\end{proposition}
This implies in particular that the composition of any extraction operator gives again an extraction operator, which is not the case in general. In particular, this fact can constitute the basis for an alternative Extraction axiom for the domain-free information algebra: 

\begin{enumerate}
\item[4] \textit{Commutative Extraction:} For all $x,y \in Q$, $\epsilon_y(\epsilon_x(\desirs)) = \epsilon_y(\epsilon_{x \wedge y}(\epsilon_x(\desirs))) = \epsilon_{x \wedge y}(\desirs).$
\end{enumerate}
Indeed, the old extraction axiom can be recovered from this one, since $x \vee z \bot y \vee z \vert z$ (equivalent to $x \bot y \vert z$) implies $(x \vee z) \wedge (y \vee z) = z$.
In the labeled case, it turns out
that it is sufficient to consider transport only for $y \le x$, if $x$ is a support of $\desirs$. Thus,
transport becomes projection or marginalization (see \citep{kohlas17}).

Such information algebras will be called \textit{commutative} (see \citep{kohlas03,kohlasschmid16}). 

According to the Commutative extraction axiom, the set of all extraction operators, that we denote with $E$, is closed under composition, hence $(E; \circ)$ form an idempotent,
commutaitive semigroup. One might replace the Commutative extraction axiom also by the requirement
that the extraction operators form an idempotent, commutative semigroup (see \cite{kohlasschmid16}).
The Combination axiom can also be simplified a bit. Let us consider the labeled case. We have:
\begin{equation*}
t_x((\desirs_1,x) \cdot (\desirs_2,y)) = (\epsilon_x(\desirs_1 \cdot \desirs_2),x)= (\desirs_1 \cdot \epsilon_x(\desirs_2),x).
\end{equation*}
Now, from commutative extraction axiom, we have $\epsilon_x(\desirs_2)=\epsilon_x(\epsilon_y(\desirs_2))= \epsilon_{x \wedge y}(\desirs_2)$. Hence, we have:
\begin{equation*}
   (\desirs_1 \cdot \epsilon_x(\desirs_2),x)= (\desirs_1 \cdot \epsilon_{x \wedge y}(\desirs_2),x)= (\desirs_1,x) \cdot t_{x \wedge y}(\desirs_2,y). 
\end{equation*}
The old axiom can then be recovered using the fact that if $x \in Q$ is a support of $\desirs$ then it is also a support of $\epsilon_y(\desirs)$ for every $y\in Q$.
These simplifications lead to the original axiomatic definition of a labeled information algebra proposed in \citep{kohlas03, shafershenoy90}.
Then, the domain-free version can be reconstructed from the labeled one. 
The general axioms of our paper can be also reconstructed from the classical multivariate version (see \cite{kohlas03, kohlas17}).

\section{Conclusions}
This  paper  presents  a  first  approach  to  information  algebras  related  to  desirable gambles on a possibility set that is not necessarily a multivariate possibility set.

There  are  many  aspects  and  issues  which are  not  addressed  here.   Foremost  is  the  issue  of  conditioning and its relations with model revision for information algebras, which in turn can be seen as the combination of the previous and new information conveyed by elements of the algebra, and belief revision for \emph{belief structures} (see \cite{decooman05}).

Further, connections with lower and upper previsions (see \cite{troffaes2014}) are not considered in this paper, as well as relationships with strictly desirable sets of gambles and almost desirable sets (see \cite{walley91}).

Finally, we would like also to analyze a particular type of information algebra, called \emph{set algebra}, that can be considered the archetype of information (see \cite{kohlas03}). In particular, we would like to show, as in the multivariate case \cite{kohlas21}, that subsets of $\Omega$ with intersection as combination and \emph{saturation operators} as extraction operators, form an example of this algebra that moreover, can be embedded in $\Phi$. This constitute also a first step to show that $\Phi$ itself, in this more general case, can be embedded into the set algebra of subsets of $At(\Phi)$.
All these subjects should be objective of future studies.

\appendix

\section{}\label{apd:first}

\begin{proof}[Proof of Theorem~\ref{th:QSepOfPart}]

C1 and C2 are obvious. To prove C3 assume $\partit_1 \bot \partit_2 \vert \partit$ and $\partit' \leq \partit_2$. Then $u \equiv_\partit u'$ implies the existence of an element $v$ such that $u \equiv_{\partit_1 \vee \partit} v$ and $u' \equiv_{\partit_2 \vee \partit} v$. But $\partit' \leq \partit_2$ means that $u' \equiv_{\partit_2 \vee \partit} v$ implies $u' \equiv_{\partit' \vee \partit} v$, and this means that $\partit_1 \bot \partit' \vert \partit$. Similarly, $u \equiv_\partit u'$ implies the existence of an element $v$ such that $u \equiv_{\partit_1 \vee \partit} v$ and $u' \equiv_{\partit_2 \vee \partit} v$, says also that $\partit_1 \bot \partit_2 \vee \partit \vert \partit$, hence C4.
\end{proof}

\begin{proof}[Proof of Theorem~\ref{th:CommCondIndep}]

We show first that $\partit_1 \bot \partit_1 \vert \partit$ implies $\partit_1 \leq \partit$. Consider blocks $B_1$, $B'_1$ of $\partit_1$ and $B$ of $\partit$. Then $B_1 \cap B \not= \emptyset$ and $B'_1 \cap B \not= \emptyset$ imply $B_1 \cap B'_1 \cap B \not= \emptyset$. But, it is possible only if $B_1=B'_1$. 
So $B$ cannot intersect two different blocks of $\partit_1$, hence $B$ must be a subset of some block of $\partit_1$ and thus $\partit_1 \leq \partit$.

Let us now divide the proof in cases.
\begin{itemize}
    \item Assume $\partit_1,\partit_2$ commute.
    If $\partit_1 \bot \partit_2 \vert \partit$, then $\partit_1 \vee \partit \bot \partit_2 \vee \partit \vert \partit$. Let
\begin{equation*}
   \partit' \coloneqq (\partit_1 \vee \partit) \wedge (\partit_2 \vee \partit). 
\end{equation*}
Then we have $\partit' \leq \partit_1 \vee \partit$ and $\partit' \leq \partit_2 \vee \partit$. Using C3,C4 and C2 we conclude that $\partit' \bot \partit' \vert \partit$ and it thus follows that $\partit' \leq \partit$. Since we have always $\partit' \geq \partit$ it follows that $\partit' = \partit$. That is one direction of the implication claimed in the theorem.

Now, assume that $(\partit_1 \vee \partit) \wedge (\partit_2 \vee \partit)=\partit$. Let us define $\partit'_1 \coloneqq (\partit_1 \vee \partit)$ and $\partit'_2 \coloneqq (\partit_2 \vee \partit)$.
Consider $\omega \equiv_\partit \omega'$. Then, by hypothesis, we have $ \omega \equiv_{\partit'_1 \wedge \partit'_2} \omega'$. This, by definition, means that there exists $\omega^{''}$ such that $\omega \equiv_{\partit'_1} \omega^{''}$ and $\omega'\equiv_{\partit'_2}\omega^{''}$. Hence $\partit_1 \bot \partit_2 \vert \partit$.

\item Now assume that $\partit_1 \bot \partit_2 \vert \partit \iff (\partit_1 \vee \partit) \wedge (\partit_2 \vee \partit)=\partit$.
Hence, given that $(\partit_1 \vee (\partit_1 \wedge \partit_2)) \wedge (\partit_2 \vee (\partit_1 \wedge \partit_2)) = \partit_1 \wedge \partit_2$, we have $\partit_1 \bot \partit_2 \vert \partit_1 \wedge \partit_2$.

Then, if $B_1$, $B_2$ and $B$ are blocks of $\partit_1$, $\partit_2$ and $\partit_1 \wedge \partit_2$ respectively, we have that $B_1 \cap B \not= \emptyset$ and $B_2 \cap B \not= \emptyset$ imply $B_1 \cap B_2 \cap  B \not= \emptyset$. Since $\partit_1 \wedge \partit_2 \leq \partit_1,\partit_2$ it follows that $B_1,B_2 \subseteq B$ and $B_1 \cap B_2 \not= \emptyset$. This means that $\partit_1$ and $\partit_2$ commute.
\end{itemize}
\end{proof}

\begin{proof}[Proof of Theorem~\ref{th:CombExtrProp}]

If $\desirs = 0$ this is obvious. So, assume $\desirs \not= 0$. Let
\begin{align*}
A \coloneqq&  \epsilon_{y \vee z}(\desirs)  = \mathcal{C}(\desirs \cap \mathcal{L}_{y \vee z}), 
\\
B \coloneqq&  \epsilon_{y \vee z}(\epsilon_z(\desirs)) =    \mathcal{C}( \mathcal{C}(\desirs \cap  \mathcal{L}_z) \cap  \mathcal{L}_{y \vee z}).
\end{align*}
Then $\desirs \cap  \mathcal{L}_z \subseteq \desirs$ implies $B \subseteq A$. 
Therefore, consider a gamble $f \in A$ so that $f \geq f'$ for a gamble $f' \in \desirs \cap \mathcal{L}_{y \vee z}$. Then we have, since $\desirs = \mathcal{C}(\desirs \cap \mathcal{L}_{x})$,
\begin{equation*}
f' \geq g, \quad g \in \desirs \cap \mathcal{L}_x,  \quad f' \textrm{ is}\ y \vee z- \textrm{measurable}.
\end{equation*}
Define for all $\omega \in \Omega$,
\begin{equation*}
g'(\omega) \coloneqq \sup_{\omega' \equiv_{y \vee z} \omega} g(\omega').
\end{equation*}
Since $f'$ is $y \vee z$-measurable, we have  $f' \geq g'$, and also $g' \in \desirs$. 
We claim that $g'$ is $z$-measurable. Indeed, consider a pair of elements $\omega \equiv_z \omega''$ and the block $B_z$ of partition $\partit_z$ which contains these two elements. Then consider the blocks $B_{y \vee z} \subseteq B_z$ and $B''_{y \vee z} \subseteq B_z$ which contain elements $\omega$ and $\omega''$ respectively. Finally consider the family of all blocks $B_{x \vee z} \subseteq B_z$. From $x \vee z \bot y \vee z \vert z$ we conclude that $B_{y \vee z} \cap B_{x \vee z} \not= \emptyset$ and $B''_{y \vee z} \cap B_{x \vee z} \not= \emptyset$ for all blocks $B_{x \vee z} \subseteq B_z$. Since $g$ is $x \vee z$-measurable, $g$ is constant on any of these blocks. Define $g(B_{x \vee z}) = g(\omega)$ if $\omega \in B_{x \vee z}$. Then it follows that
\begin{equation*}
g'(\omega) = \sup_{B_{x \vee z} \cap B_{y \vee z} \neq \emptyset} g(B_{x \vee z}) =\sup_{B_{x \vee z} \subseteq B_z} g(B_{x \vee z}),
\end{equation*}
and
\begin{equation*}
g'(\omega'')= \sup_{B_{x \vee z} \cap B''_{y \vee z} \neq \emptyset} g(B_{x \vee z}) = \sup_{B_{x \vee z} \subseteq B_z} g(B_{x \vee z}).
\end{equation*}
This shows that $g'$ is $z$-measurable, hence $g' \in \desirs \cap \mathcal{L}_z$. So we conclude that $f' \in \mathcal{C}(\desirs \cap \mathcal{L}_z) \cap \mathcal{L}_{y \vee z}$. But this implies that $f \in B$ and so $A = B$. This concludes the proof.
\end{proof}

\begin{proof}[Proof of Proposition~\ref{prop:finersupports}]
Assume that $\desirs = \mathcal{C}(\desirs \cap \mathcal{L}_x)$. If $x \leq y$, then $\mathcal{L}_x \subseteq \mathcal{L}_y$, hence $\desirs \cap \mathcal{L}_x \subseteq \desirs \cap \mathcal{L}_y$. It follows that $\desirs = \mathcal{C}(\desirs \cap \mathcal{L}_x) \subseteq \mathcal{C}(\desirs \cap \mathcal{L}_y)$. But $\mathcal{C}(\desirs \cap \mathcal{L}_y) \subseteq \desirs$, hence $\mathcal{C}(\desirs \cap \mathcal{L}_y) = \desirs$ and so $\epsilon_y(\desirs)  = \desirs$.
\end{proof}

\begin{proof}[Proof of Theorem~\ref{th:GenCondIndep}]

From $x \bot y \vert z$ it follows by the properties of a quasi separoid that $x \vee z \bot y \vee z \vert z$. Therefore by the Extraction axiom,
\begin{equation*}
\epsilon_{y \vee z}(\desirs) = \epsilon_{y \vee z}(\epsilon_z(\desirs)).
\end{equation*}
By Lemma \ref{le:SuppotProp}, we have $\epsilon_y(\epsilon_{y \vee z}(\desirs)) = \epsilon_y(\desirs)$ and $\epsilon_y(\epsilon_{y \vee z}(\epsilon_z(\desirs)) = \epsilon_y(\epsilon_z(\desirs))$. This proves the first part. 

By the Existential quantification and the Extraction axioms, we have,
\begin{equation*}
\epsilon_{y \vee z}(\desirs_1 \cdot \desirs_2) = \epsilon_{y \vee z}(\desirs_1) \cdot \desirs_2 = \epsilon_{y \vee z}(\epsilon_z(\desirs_1)) \cdot \desirs_2,
\end{equation*}
because, by the Support axiom, $y \vee z$ is a support of $\desirs_2$.

By Lemma \ref{le:SuppotProp}, the last combination equals $\epsilon_z(\desirs_1) \cdot \desirs_2$. 
But then, again by the Existential quantification axiom and Lemma \ref{le:SuppotProp},
\begin{align*}
\epsilon_z(\desirs_1 \cdot \desirs_2) &= \epsilon_z(\epsilon_{y \vee z}(\desirs_1 \cdot \desirs_2)) = \\
\epsilon_z(\epsilon_z(\desirs_1) \cdot \desirs_2) &= \epsilon_z(\desirs_1) \cdot \epsilon_z(\desirs_2).
\end{align*}
This concludes the proof.
\end{proof}

\begin{proof}[Proof of Theorem~\ref{th:epsilonX_atomistic}]

Let $M \in At(\Phi)$. Then $M \not= 0$ and thus $\epsilon_x(M) \not= 0$. Assume 	$\epsilon_x(M) \leq \epsilon_x(\desirs)$ for some $\desirs \in \Phi$. Then $\epsilon_x(M \cdot \epsilon_x(\desirs)) = \epsilon_x(M) \cdot \epsilon_x(\desirs) = \epsilon_x(\desirs)$. Since $M$ is an atom, we have either $M \cdot \epsilon_x(\desirs) = M$ or $\epsilon_x(\desirs) = 0$. In the first case $\epsilon_x(\desirs) = \epsilon_x(M \cdot \epsilon_x(\desirs)) = \epsilon_x(M)$. So $\epsilon_x(M)$ is an atom in $\epsilon_x(\Phi)$.

Further, if $0 \not= \epsilon_x(\desirs)$, then, since $\Phi$ is atomic, there is an atom $M$ such that $\epsilon_x(\desirs) \leq M$ and thus $\epsilon_x(\desirs) \leq \epsilon_x(M)$. As shown, $\epsilon_x(M)$ with $M \in At(\Phi)$, is an atom in $\epsilon_x(\Phi)$, and the subalgebra $\epsilon_x(\Phi)$ is atomic.

Now, suppose again $\epsilon_x(\desirs) \neq 0$. By 
the fact that $\epsilon_x(\desirs) \in \Phi$ with $\Phi$ atomistic, we have:
\begin{equation*}
\epsilon_x(\desirs) = \bigcap  At(\epsilon_x(\desirs)) =
\bigcap \{M:M \in At(\epsilon_x(\desirs))\})
.
\end{equation*}
But, $\epsilon_x(\desirs) \subseteq M$ holds if and only if $\epsilon_x(\desirs) \subseteq \epsilon_x(M)$. So, if we define $At_x(\epsilon_x(\desirs))$ to be the set of all atoms in $\epsilon_x(\Phi)$ dominating $\epsilon_x(\desirs)$, we obtain $\epsilon_x(\desirs) = \bigcap At_x(\epsilon_x(\desirs))$. This shows the atomicity of the subalgebra $\epsilon_x(\Phi)$.
\end{proof}

\begin{proof}[Proof of Proposition~\ref{prop:CommExtrOp}]
We have $x \bot y \vert x \wedge y$ for all $x,y \in Q$ since the partitions commute. Then, by item 1 of Theorem \ref{th:GenCondIndep} and items 5 and 6 of Lemma \ref{le:SuppotProp},
\begin{equation*}
\epsilon_y(\epsilon_x(\desirs))
=\epsilon_{y }(\epsilon_{x \wedge y}(\epsilon_x(\desirs)))=\epsilon_{x \wedge y}(\epsilon_x(\desirs)= \epsilon_{x \wedge y}(\desirs).
\end{equation*}
In the same way we obtain $\epsilon_x(\epsilon_y(\desirs)) = \epsilon_{x \wedge y}(\desirs)$. 
\end{proof}

\bibliography{isipta2021-template}

\end{document}